\documentclass{article}
\usepackage[a4paper, margin=3cm]{geometry}
\usepackage{tabularx}
\usepackage{amsmath}
\usepackage{amsfonts}
\usepackage{amssymb}
\usepackage{setspace}
\usepackage{fontawesome5}
\usepackage{xcolor}
\usepackage{graphicx}
\usepackage{algorithm}
\usepackage{algpseudocode}
\definecolor{linkblue}{HTML}{0645AD}
\usepackage[
    colorlinks=true,
    linkcolor=linkblue,
    citecolor=linkblue,
    urlcolor=linkblue,
    filecolor=linkblue
]{hyperref}
\usepackage{amsthm}
\theoremstyle{plain}
\newtheorem{thm}{Theorem}[section]
\theoremstyle{definition}

\theoremstyle{remark}

\usepackage[backend=biber,
    style=authoryear-comp,
    maxbibnames=99,
    maxcitenames=2,
    uniquename=false,
    uniquelist=false,
    giveninits=true,
    ]{biblatex}
\definecolor{orcidlogocol}{HTML}{A6CE39}
\newcommand{\orcid}[1]{\href{https://orcid.org/#1}{\textcolor{orcidlogocol}\faOrcid}}

\newenvironment{keypoints}{
    \small
    \vspace{0.3em}
    \noindent\hspace{1.5em}\textbf{Key Points}
    \vspace{-0.4em}
    \begin{itemize}
        
        \setlength\itemsep{0.15em}
        \setlength\parskip{0pt}
        \setlength\parsep{0pt}
}{
    \end{itemize}
    \vspace{0.3em}
}

\title{Probabilistic Guarantees for Reducing Contextual Hallucinations in LLMs} 
\author{
    \begin{tabular}{cc}
        Nils Rautenberg\footnote{contact: nils.rautenberg@rub.de} & Sven Schippkus~\orcid{0000-0002-8504-6811} \\
        Deutsche Aktuarvereinigung e.V. & University of Hamburg \\
    \end{tabular}
}
\date{\today}

\begin{document}

\setstretch{1.25}

\maketitle

\begin{keypoints}
\item Independent prompt repetition makes it overwhelmingly likely that a correct answer appears.  
\item An LLM-as-a-judge identifies correct answers; majority vote improves reliability when needed.  
\item Together, these give explicit probabilistic guarantees with exponentially decreasing error rates.
\end{keypoints}

\begin{abstract}
Large language models (LLMs) frequently produce contextual hallucinations, where generated content contradicts or ignores information explicitly stated in the prompt. Such errors are particularly problematic in deterministic automation workflows, where inputs are fixed and correctness is unambiguous. We introduce a simple and model‑agnostic framework that provides \emph{explicit probabilistic guarantees} for reducing hallucinations in this setting.  

We formalize the notion of a \emph{specific task}, defined by a fixed input and a deterministic correctness criterion, and show that issuing the same prompt in independent context windows yields an exponential reduction in the probability that all model outputs are incorrect. To identify a correct answer among repeated runs, we incorporate an LLM-as-a-judge and prove that the probability that the judged pipeline fails decays at a rate determined by the judge's true- and false-positive probabilities. When the judge is imperfect, we strengthen it through majority vote over independent judge calls, obtaining ensemble-level error rates that decrease exponentially in the number of votes. This yields an explicit bound on the probability that the pipeline selects a hallucinated answer.

Experiments on controlled extraction tasks with synthetic noisy judges match these predictions exactly: pipeline failure decreases exponentially with the number of repetitions, and hallucination-selection decreases exponentially with the number of judges in the ensemble. Together, these results provide a lightweight, modular, and theoretically grounded method for driving hallucination probabilities arbitrarily low in fixed-input LLM workflows-without modifying model weights, decoding strategies, or prompt engineering.
\end{abstract}

\section{Introduction}

Large language models (LLMs) are increasingly capable of solving complex tasks, following instructions, and synthesizing information across long chains of reasoning. Yet even the most advanced models still fail in a surprisingly human way: they occasionally assert things that are not true, not implied, or directly contradicted by the information they were given. These hallucinations pose a challenge for automated pipelines that require deterministic correctness.

A growing body of work has shown that hallucinations take many forms rather than one. Comprehensive reviews \parencite{huang2025, zhang2025} map out diverse error types—from contradictions with world knowledge to logical slips or instruction confusion—revealing that hallucination is a structural property of present-day models rather than a rare glitch \parencite{kalai2025}. Summarization studies further demonstrate that even conditioned generation systems often produce unsupported claims \parencite{maynez2020, fabbri2021, kryscinski2020}, and retrieval-augmented models continue to introduce content not present in the retrieved passages \parencite{adlakha2024}. Contextual hallucinations, where the model contradicts or ignores information that is explicitly provided in the prompt, have proven particularly difficult to eliminate. Prior work attributes these failures to long‑context degradation \parencite{chang2024, kim2024, liu2025}, competition between contextual and parametric knowledge \parencite{longpre2021}, and the inherent difficulty of reliable evidence tracking in autoregressive generation.

\begin{figure}[t]
    \centering
     \makebox[0pt][c]{
    \includegraphics[width=18cm]{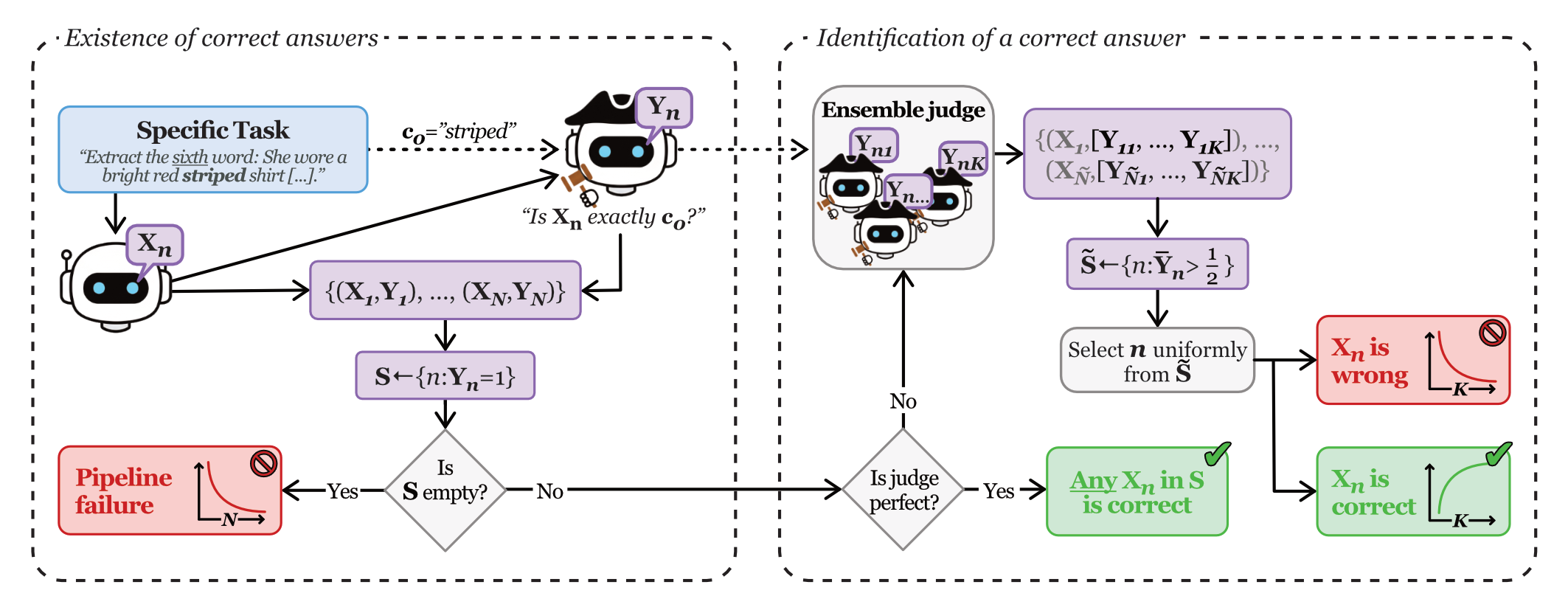}
    }
    \caption{The repetition-judge pipeline gives probabilistic guarantees for reducing contextual hallucinations. \textbf{Left:} For a given \emph{specific task}, the likelihood that all answers are hallucinated is reduced by repeating the task $N$ times and judging the output $X_n$ of each for correctness. The pipeline fails ("all answers are judged as incorrect") exponentially decreasing with $N$. \textbf{Right:} From the judged-correct answers, a true correct answer can be found at guaranteed rates. If the judge is perfect ($q^{++}=1,\, q^{-+}=0$), any $X_n$ with $Y_n=1$ is a correct answer. If the judge is noisy but better than random ($q^{-+}<0.5$), an ensemble judge selects a hallucination exponentially decreasing with the number of judges $K$.}
    \label{fig:sketch}
\end{figure}

Ongoing efforts have sought to detect and mitigate hallucinations. Techniques such as FactCC \parencite{kryscinski2020}, AlignScore \parencite{zha2023}, REFChecker \parencite{hu2024}, and VeriFastScore \parencite{rajendhran2025} attempt to identify unsupported content post‑hoc, while contrastive decoding strategies \parencite{shi2023, chang2023} and faithfulness-inspired heuristics \parencite{wan2023} aim to reduce hallucinations during inference. Multi-agent refinement and LLM-as-a-judge systems \parencite{jain2023, luo2023, cohen2023, liu2023, madaan2023, shinn2023, du2023, wang2024} combine several model calls to improve reliability, but empirical studies show that judges can themselves be inconsistent or biased in challenging cases \parencite{laban2023, kim2024}. Despite substantial progress, existing mitigation strategies do not provide explicit quantitative guarantees on the likelihood of residual hallucinations. 

In this paper, we focus on a restricted but practically important setting: a class of tasks that we call \emph{specific tasks}. A specific task is one where: (i) the input is fixed, (ii) the prompt is issued only once per run, (iii) and correctness can be evaluated by a deterministic criterion. Typical examples include slot extraction, reference lookup, formatting tasks, or locating a particular word or number in a known text. Such tasks are common in automated workflows and provide a clean setting where correctness is unambiguous. In this setting, reliability decomposes naturally into two independent questions: (i) how likely is that at least one correct answer appears among the model outputs, and (ii) how reliably a judge can identify such an answer when it exists. These two stages-existence and identification-can be improved separately through repetition and judged selection.

Two simple observations motivate our framework:

\begin{enumerate}
    \item \textbf{Repetition amplifies correctness.}  
    Even if a model sometimes answers incorrectly, running the same task several times independently, for example by issuing each call in a fresh context window, makes it very likely that at least one of the answers will be correct. This is a direct probabilistic consequence of independence, and we formalize it as the \emph{Repetition Lemma}.
    
    \item \textbf{Identification requires a judge.}  
    Repetition alone does not tell us \emph{which} of the responses is correct. Selecting an answer uniformly at random from the repeated runs yields the same correctness probability as a single run. To select a correct answer among all that appear, we need a secondary mechanism to filter out hallucinations: an \emph{LLM-as-a-judge}. Even if the judge is imperfect, its reliability can itself be improved through repetition.
\end{enumerate}

Figure~\ref{fig:sketch} illustrates the repetition–judge pipeline based on these core concepts. Repetition makes it overwhelmingly likely that at least one correct answer appears among the model outputs, and an appropriate judge–strengthened via majority vote if necessary—identifies such an answer with task‑dependent true‑ and false‑positive rates. Taken together, these components form a lightweight, model‑agnostic pipeline for achieving arbitrarily low error rates on fixed‑input tasks without modifying model weights or relying on complex decoding schemes.

The remainder of the paper develops these ideas in a simple formal setting. Section \ref{sec:existence} introduces \emph{specific tasks} and shows how independent repetitions sharply reduce the chance that all responses are wrong.  Section \ref{sec:identification} then explains how a judge can be used to identify a correct answer once repetition makes its existence likely, and how ensemble judging provides explicit bounds on hallucination-selection probability even for imperfect judges. Section \ref{sec:experiment} presents controlled experiments validating both effects, and Section \ref{sec:conclusion} concludes.

\section{Existence of correct answers}
\label{sec:existence}

At the heart of our framework lies a simple but powerful modelling assumption: independent LLM calls made in fresh context windows behave as independent samples from a fixed output distribution. This assumption is the mathematical basis for all existence guarantees in this paper. It reflects the practical reality that each model call carries its own sampling randomness and that starting from an empty, uncontaminated context eliminates cross‑run interference. Under this independence assumption, repeated executions of the same prompt provide statistically independent “attempts” at solving a fixed task.

We now formalize the setting in which our analysis applies. At this stage, we focus solely on the probability that a correct answer \emph{appears} among repeated calls to the model. The identification step ("find a correct answer") is developed separately in Section~\ref{sec:identification}, where we incorporate a judge and, when needed, majority‑vote aggregation.

A \emph{specific task} is one where the input is fixed, the prompt is issued only once per run, and correctness can be evaluated by a deterministic rule. In this setting, each call to the LLM behaves as a random draw from a fixed distribution. Running the same task independently multiple times therefore produces a sequence of independent and identically distributed random variables. This structure is all we need to derive explicit probabilistic guarantees on the likelihood that at least one correct answer appears among repeated executions.

\subsection{Specific tasks}

Let Z denote the set of possible LLM inputs and outputs. A specific task consists of:

\begin{itemize}
    \item a fixed input $A \in Z$,
    \item a fixed prompt function $B: Z \rightarrow Z$ producing the query to the LLM,
    \item a deterministic correctness function $C: Z \rightarrow Z$,
    \item and a correctness criterion $c_0 \in Z$.
\end{itemize}

One execution of the task consists of issuing the prompt $B(A)$ to the LLM in a \emph{fresh context window} and obtaining an output $X \in Z$. We say the model answers correctly if $C(X)=c_0$, and we call the response a \emph{hallucination} otherwise.

Crucially, issuing each call in a fresh context window makes the outputs $X_1, X_2, \dots$ behave as independent and identically distributed random variables. This is the only, yet powerful, structural property we assume about the model. Because the input, prompt, and correctness criterion are fixed, the per-run success probability
\[p=\mathbb{P}(C(X_n)=c_0),\]
is a stable property of the task. A task is \emph{solvable} if $p>0$.

\subsection{Repetition Lemma}

We are interested in the probability that $N$ independent model calls all fail to produce a correct answer. Since each call succeeds with probability $p$ and fails with probability $(1-p)$, independence gives
\[\mathbb{P}(C(X_1)\neq c_0,\ldots,C(X_N)\neq c_0) = (1-p)^N.\]

This yields the following simple but powerful observation.

\begin{thm}[Repetition Lemma]
\label{thm:repetition}
For any solvable specific task with per‑run correctness probability $p>0$, the probability that $N$ independent executions all produce hallucinated responses is
\[(1-p)^N.\]
This probability decreases exponentially with increasing $N$.
\end{thm}

\begin{proof}
    Since the events $\{C(X_n)\neq c_0\}$ are independent and each has probability $(1-p)$, the joint failure probability is exactly $(1-p)^N$.
\end{proof}

Intuitively, if the model occasionally produces the correct answer, then repeatedly running the same task makes it overwhelmingly likely that a correct answer appears somewhere among the responses. The lemma provides an explicit quantitative guarantee: a small amount of repetition is sufficient to make the event ``every answer is a hallucination'' extremely unlikely. Repetition therefore provides a guarantee on the \emph{existence} of a correct answer among the repeated outputs. The problem of \emph{identifying} such an answer is addressed next.

\section{Identification of correct answers}
\label{sec:identification}

Repetition ensures that a correct answer is likely to appear among the repeated outputs. However, repetition alone does not determine \emph{which} of the responses is correct. Selecting a response uniformly at random from the repeated runs yields the same correctness probability as a single model call. To identify a correct answer among the repeated outputs, we use an LLM-as-a-judge, strengthened through majority vote if necessary. This section introduces the judge, derives the probability that the pipeline fails to identify any correct answer, and provides an explicit bound on the probability that the selected answer is a hallucination.

\subsection{Judging model}

For each model output $X_n$, we issue a separate yes/no query to a judge asking whether $X_n$ is correct with respect to the fixed criterion $C$, that is $C(X_n)=c_0$. Each query is sent in a fresh context window so that judge evaluations across different $n$ are independent. We write the judge response as a Bernoulli random variable
\[Y_n \in \{0,1\},\quad Y_n=1~\text{meaning ``the judge asserts }X_n\text{ is correct.''}\]
To characterise the judge, we define is true-positive and false-positive rates
\[q^{++}=\mathbb{P}(Y_n = 1 \mid C(X_n) = c_0),\quad q^{-+}=\mathbb{P}(Y_n = 1 \mid C(X_n) \neq c_0).\]
These quantities depend on the task and judge prompt but are fixed across repetitions. A perfect judge satisfies $q^{++}=1$ and $q^{-+}=0$. We emphasise that judge evaluation itself constitutes a specific task with fixed correctness criterion, and thus inherits the same independence assumptions as the model calls.

\subsection{Reliability of judged existence}

The judged-selection pipeline fails if, and only if, no model output is judged correct. For any given model call, two events must occur to obtain a correct-and-judged-correct output: (i) the model produces a correct answer, which happens with probability $p$, and (ii) the judge accepts it, which happens with probability $q^{++}$. Thus the per-trial probability of producing a correct-and-judged-correct answer is $pq^{++}$.

Similarly, an incorrect answer is nevertheless judged correct with probability $(1-p)q^{-+}$, producing a false-positive candidate. Although such false positives do not help the pipeline identify a correct answer, they influence the behaviour of the judged-correct set and therefore appear in the full derivation of pipeline failure. 

Putting these observations together yields the following guarantee.

\begin{thm}[Reliability of the repetition-judge pipeline]
\label{thm:pipeline-failure}
For a specific task with per‑trial correctness probability $p>0$ and judge rates $q^{++}>0\,,q^{-+}\geq0$, the probability that the pipeline fails, i.e., that no judged-correct answer appears among N repetitions, is
\[\bigl(1-(pq^{++}+(1-p)q^{-+})\bigr)^N.\]
\end{thm}

\begin{proof}
    For any fixed $n$, the probability that the output is judged correct (whether correctly or incorrectly) is $pq^{++}+(1-p)q^{-+}$. Independence across repetitions implies that the probability that no judged-correct answer appears across all $N$ trials is the stated expression. This event contains, in particular, the event that no correct-and-judged-correct answer appears.
\end{proof}

Let 
\[S=\{n:Y_n=1\}\]
denote the set of responses that the single judge asserts as correct. Theorem~\ref{thm:pipeline-failure} characterizes the probability that $S=\varnothing$, i.e., that the pipeline reports failure. From now on, we condition on the natural event $S\neq\varnothing$, meaning that at least one model output was judged correct by the pipeline.

\subsection{Identification via ensemble judging}
Given the repetition-judge pipeline succeeds ($S\neq\varnothing$), we can define bounds on how likely we find a correct answer in $S$, depending on the judge quality. In the following, we investigate the case where the judge is better than random ($q^{++}>{1\over 2}$, $q^{-+}<{1\over 2}$).

For simple assessment tasks, e.g., when $c_0$ is a single word and we ask the judge ``Is $X_n$ exactly $c_0?$'', an LLM-as-a-judge is often perfect in practice, meaning $q^{++}=1$ and $q^{-+}=0$. In this case, the pipeline failure rate (Theorem~\ref{thm:pipeline-failure}) simplifies to the Repetition Lemma (Theorem~\ref{thm:repetition}) and when the pipeline succeeds ($S\neq\varnothing$), selecting any answer guarantees a correct answer.

However, if $q^{-+}$ is non-zero, a hallucinated answer may appear in the judged set $S$ and could be chosen. If the judge is better than random ($q^{-+}<{1\over 2}$), we can use an ensemble of judges and issue a majority vote among them to reduce the probability of selecting a hallucination.

For each $X_n$, let $Y_{n1},\,...\,,Y_{nK}$ be independent Bernoulli variables with
\[\mathbb{P}(Y_{nk}=1 \mid C(X_n)=c_0)=q^{++},\quad \mathbb{P}(Y_{nk}=1 \mid C(X_n)\neq c_0)=q^{-+}.\]
Thus every judge in the ensemble has the same quality as the original base judge. This is not required for the guarantees we derive in the following but it simplifies the implementation in practice, as the same LLM-as-a-judge can be re-used for all judging tasks.

The ensemble judge accepts $X_n$ if a strict majority of the $K$ votes are positive, i.e., 
\[\stackrel{\sim}{S} = \Bigg\{n:\frac{1}{K}\sum_{k=1}^KY_{nk}>{1\over 2}\Bigg\} = \bigg\{n:\overline{Y}_{n}>{1\over 2}\bigg\}.\]
We denote the ensemble true- and false-positive rates by
\[
Q^{++}(K) = \mathbb{P}\bigg(\overline{Y}_{n}>{1\over 2}\mid C(X_n)=c_0\bigg),\quad 
Q^{-+}(K) = \mathbb{P}\bigg(\overline{Y}_{n}>{1\over 2}\mid C(X_n)\neq c_0\bigg).\]

In the selection stage, the final answer is chosen uniformly from $\stackrel{\sim}{S}$. We therefore condition on the natural event $\stackrel{\sim}{S}\neq\varnothing$, which holds whenever the ensemble judge accepts at least one candidate. We now compute the probability that this chosen answer is a hallucination.
\begin{thm}[Hallucination-selection with an ensemble judge]
\label{thm:ensemble-hallucination}
Let
\begin{itemize}
    \item $p$ be the per-trial model correctness probability,
    \item $q^{++}$ and $q^{-+}$ be the true- and false-positive rates of the \emph{base} judge, and
    \item $Q^{++}(K)$ and $Q^{-+}(K)$ be the corresponding true- and false-positive rates of the \emph{ensemble} judge formed from $K$ independent votes.
\end{itemize}
Because the ensemble judge is applied only to those outputs that the base judge has first accepted, only candidates in $S=\{n:Y_n=1\}$ are forwarded to the ensemble and thus a correct model answer is accepted by the ensemble judge with probability
\[pq^{++}Q^{++}(K).\]
Similarly, an incorrect answer is accepted with probability
\[(1-p)q^{-+}Q^{-+}(K).\]
Conditioned on $\stackrel{\sim}{S}\neq\varnothing$, the probability that a uniformly chosen element of $\stackrel{\sim}{S}$ is a hallucination is
\[\mathbb{P}(C(X_n)\neq c_0\mid\,\stackrel{\sim}{S}\neq0) = \frac{(1-p)\,q^{-+}\,Q^{-+}(K)}{p\,q^{++}\,Q^{++}(K)+(1-p)\,q^{-+}\,Q^{-+}(K)}.\]
\end{thm}

\begin{proof}
    An answer enters $\stackrel{\sim}{S}$ if and only if (i) the model produces it, (ii) the base judge marks it correct, and (iii) the majority of the ensemble judges mark it as correct.
    A correct answer therefore enters $\stackrel{\sim}{S}$ with probability
    \[p \times q^{++} \times Q^{++}(K),\]
    and a hallucinated answer enters with probability
    \[(1-p) \times q^{-+} \times Q^{-+}(K).\]
    Conditioning on $\stackrel{\sim}{S}\neq\varnothing$ and selecting uniformly from $\stackrel{\sim}{S}$, the selection probability is proportional to these entry rates, and normalisation yields the stated expression.
\end{proof}

\subsection{Ensemble judge accuracy}
The ensemble false-positive rate can be written in closed form as
\[Q^{-+}(K)=\sum_{i=\lceil K/2 \rceil}^K \binom{K}{i} (q^{-+})^i(1-q^{-+})^{K-i},\]
and similarly for $Q^{++}(K)$. Standard Hoeffding bounds give, for any $q^{-+}<\frac{1}{2}$,
\[Q^{-+}(K)\leq e^{-2K\big(\frac{1}{2}-q^{-+}\big)^2},\]
showing exponential decay with the number of judges K. Substituting these rates into Theorem~\ref{thm:ensemble-hallucination} yields an exponentially decreasing hallucination-selection probability.

\subsection{Algorithmic formulation}

The following procedure summarizes the repetition--judge pipeline for a perfect judge.

\begin{algorithm}[h!]
\caption{Repetition--Judge Pipeline (Perfect-Judge Case)}
\label{alg:pipeline}
\begin{algorithmic}[1]
\Require Specific task with fixed input $A$ and correctness criterion $C$; number of repetitions $N$
\Ensure A selected answer or \textsc{Failure}
\For{$n = 1 \dots N$}
    \State $X_n \gets$ LLM query with task prompt
    \State $Y_n \gets$ judge query in a fresh context window:
        ``Is $X_n$ exactly correct?''
\EndFor
\State $S \gets \{n : Y_n = 1\}$ \Comment{All answers judged correct by the (perfect) judge}
\If{$S = \varnothing$}
    \State \Return \textsc{Failure}
\Else
    \State Choose $n$ uniformly random from $S$
    \State \Return $X_n$
\EndIf
\end{algorithmic}
\end{algorithm}

\section{Experiments}
\label{sec:experiment}

We conduct controlled experiments to validate the two quantitative guarantees established in Theorem~\ref{thm:pipeline-failure} (``How likely are all answers hallucinated?'') and Theorem~\ref{thm:ensemble-hallucination} (``How likely will an ensemble judge select a hallucinated answer?'') using three fixed-input extraction tasks (Appendix~\ref{app:tasks}). All model outputs are generated using \verb|Qwen3‑4B‑Instruct‑2507| \parencite{yang2025}. We provide the implementation online \parencite{schippkus2026}.

Each task requires the model to return a specific word from a short prompt, making correctness unambiguous. For each task we estimate the model's per-trial correctness probability $p$ from $10,000$ independent samples. To control judge quality, we use a synthetic judge that independently flips the assigned labels $Y_n$ with probability 0.25, which yields the base judge quality rates $q^{++}\approx0.75$ and $q^{-+}\approx0.25$. For ensemble judging, we generate $K\in\{1,3,\dots,17\}$ independent judge evaluations per output and apply a majority vote.

We execute the full pipeline $10,000$ times and record (i) whether the pipeline fails (the event $S=\varnothing$) and (ii) whether a hallucination is selected from the ensemble-judged set $\stackrel{\sim}{S}$. Across all tasks, empirical pipeline failure decreases with $N$ exactly as predicted by Theorem~\ref{thm:pipeline-failure} (Figure~\ref{fig:experiments}a--c). Similarly, hallucination-selection rates are shown in Figures~\ref{fig:experiments}d--f and decrease with $K$ exactly as predicted by Theorem~\ref{thm:ensemble-hallucination}.

\begin{figure*}[t]
    \centering
    \includegraphics[width=\linewidth]{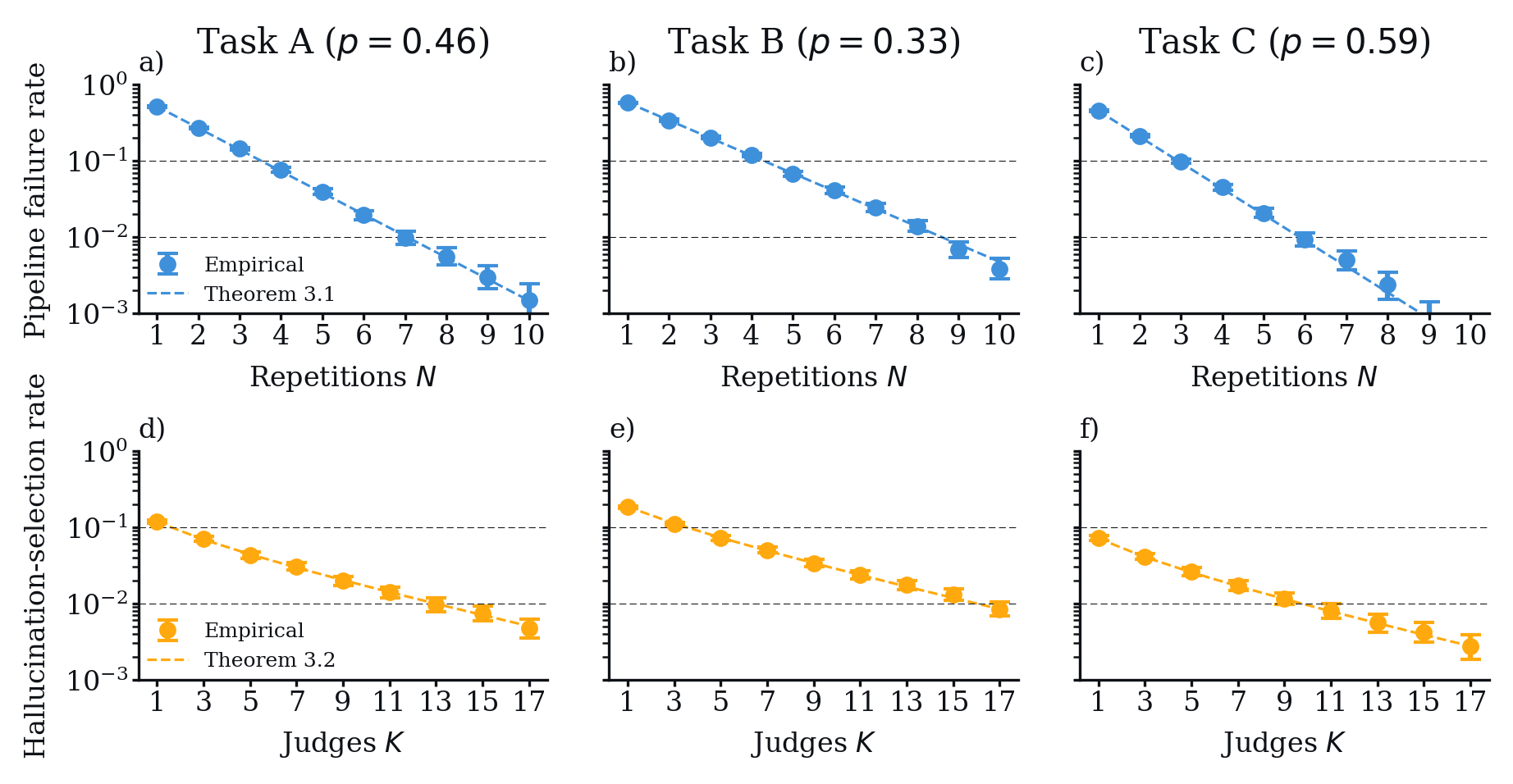}
    \caption{Empirical performance of the repetition--judge pipeline and ensemble judge, estimated over $10,000$ pipeline runs. We present three example tasks with different $p$ and use an imperfect judge ($q^{++}=0.75,\,q^{-+}=0.25$).
    \textbf{a--c)} Pipeline failure rate ("all answers are judged as incorrect") decreases exponentially with the number of repetitions $N$. Theorem~\ref{thm:pipeline-failure} predicts the observed decay, indicated by the dashed line.
    \textbf{d--f)} Hallucination-selection rate ("the ensemble judge selects a hallucinated answer") decreases exponentially with ensemble size $K$. Theorem~\ref{thm:ensemble-hallucination} predicts the observed decay, indicated by the dashed line. All error bars indicate Wilson confidence intervals. Together, these demonstrate that the repetition-judge pipeline lays the foundation for reducing contextual hallucinations to desired rates, even for imperfect judges.}
    \label{fig:experiments}
\end{figure*}

\section{Conclusion}
\label{sec:conclusion}

This paper presents a probabilistic framework that provides explicit reliability guarantees for fixed‑input LLM workflows. We focus on \emph{specific tasks}, where the prompt and correctness criterion are fixed, and show that reliability separates naturally into two independent components: ensuring that a correct answer \emph{appears}, and reliably \emph{identifying} it. We show how under the assumption of independence between context windows both can be solved.

For existence, independent repetitions of the same prompt in fresh context windows reduce the probability that all model outputs are incorrect at an exponential rate (Theorem~\ref{thm:repetition}). For identification, we incorporate an LLM‑as‑a‑judge. The reliability of the judged pipeline depends only on the judge’s true‑ and false‑positive rates (Theorem~\ref{thm:pipeline-failure}). When the judge is imperfect, majority voting over independent judge calls yields ensemble‑level error rates that also decrease exponentially (Theorem~\ref{thm:ensemble-hallucination}). These two components—repetition and judged selection—give explicit end‑to‑end guarantees with two independent, tunable parameters: the number of model repetitions $N$ and the number of judge calls $K$.

Experiments on controlled extraction tasks with synthetic noisy judges show that empirical behavior closely matches the theoretical predictions: pipeline failure decreases exponentially with $N$ and hallucination‑selection decreases exponentially with $K$. This demonstrates that highly reliable fixed‑input pipelines can be achieved without changing model weights, prompts, or decoding strategies. It crucially also demonstrates that our assumption of independence is appropriate, at least for the simple tasks investigated here. We encourage developing this concept further under less strict requirements.

Open directions include analyzing judge robustness across domains, optimizing the computational trade‑off between $N$ and $K$, and extending the framework to tasks with graded correctness or longer multi‑step workflows. Nonetheless, the results establish that any specific task with non‑zero success probability can, in principle, be solved to arbitrarily high reliability by combining independent repetitions with judged selection, providing a simple and theoretically grounded foundation for dependable agentic LLM systems

\subsection*{Acknowledgements}
This research was partially funded by the Federal Ministry of Education and Research (BMBF) and the Free and Hanseatic City of Hamburg under the Excellence Strategy of the Federal Government and the Länder.

\subsection*{Author contributions (CRediT)}
Conceptualization: NR, SS;
Formal analysis: NR; SS;
Funding Acquisition: SS;
Investigation: SS;
Methodology: NR, SS;
Resources: SS;
Software: SS;
Visualization: SS;
Writing – original draft: NR, SS;
Writing – review \& editing: NR, SS.

\printbibliography

\appendix
\section{Task prompts}
\label{app:tasks}

\begin{table*}[h!]
    \centering
    \begin{tabularx}{\textwidth}{l|X|l|l|l}
    &\textbf{\{task\}} & \textbf{\{target\}} & $c_0$ & \(\boldsymbol{p}\) \\ \hline
    A
    &Sentence one: She wore a bright red striped shirt to the party.
    
    Sentence two: It was very stylish.
    &sixth
    &striped
    &0.46
    \\ \hline
    B
    &Sentence one: Information about the new project was sent to all staff members.
    
    Sentence two: Please check your email.
    &seventh
    &sent
    &0.33
    \\ \hline
    C
    &Sentence one: Those three small round silver coins were found in the dirt.
    
    Sentence two: They looked like ancient Roman money.
    &eighth
    &found
    &0.59
    \\
    \end{tabularx}
    \caption{Three example tasks with \textbf{\{task\}} and \textbf{\{target\}} interpolated in the prompt template. For each task, the ground truth $c_0$ is returned with a probability of \(\boldsymbol{p}\), estimated from $10,000$ model runs.}
    \label{tab:experiment_tasks}
\end{table*}

We use the following prompt template into which we interpolate \{task\} and \{target\} as listed in Table~\ref{tab:experiment_tasks}. The resulting prompts are the three \emph{specific tasks} A, B, and C we investigate empirically (Figure~\ref{fig:experiments}). Prompt template:
\begin{quote}
``Below are two sentences.

Your task: From the *first* sentence, find the \{target\} word and respond with
only that word.

Note: The first sentence has many words, including some that are distractors. The key word is near the middle but might be obscured by distractors or complex phrasing. If the answer is clear, give the {target} word. If you  are unsure or find it too confusing, say "I don't know".

Here are the sentences:

\{task\}

Respond only with the \{target\} word of the first sentence, or "I don't know"
if you're unsure.''
\end{quote}

\end{document}